\documentclass{article}
\usepackage{ijcai19}

\usepackage{times}
\usepackage{soul}
\usepackage{url}
\usepackage[hidelinks]{hyperref}
\usepackage[utf8]{inputenc}
\usepackage[small]{caption}
\usepackage{graphicx}
\usepackage{amsmath}
\usepackage{booktabs}
\usepackage{algorithm}
\usepackage{algorithmic}
\usepackage{graphicx,color}
\usepackage{amsthm}

\newtheorem{theorem}{Theorem}
\newtheorem{lemma}[theorem]{Lemma}
\newtheorem{proposition}[theorem]{Proposition}
\newtheorem{corollary}[theorem]{Corollary}
\newtheorem{definition}{Definition}
\newtheorem{example}{Example}

\newcommand{\adom}[1]{\mathsf{terms}(#1)}
\newcommand{\I}{I}
\newcommand{\Ia}{I_a}

\newcommand{\ochase}{\mathsf{o}\textsf{-}\mathsf{chase}}
\newcommand{\soochase}{\mathsf{(s)o}\textsf{-}\mathsf{chase}}
\newcommand{\sochase}{\mathsf{so}\textsf{\hspace{-.5px}-}\mathsf{chase}}

\newcommand{\anychase}{\star\textsf{\hspace{-.5px}-}\mathsf{chase}}

\newcommand{\anychaseIS}[2]{\anychase(#1,#2)}
\newcommand{\anychasek}[3]{\anychase^{#3}(#1,#2)}
\newcommand{\anychaseshort}{\anychaseIS{\I}{\ruleset}}
\newcommand{\anychasekshort}[1]{\anychasek{\I}{\ruleset}{#1}}

\newcommand{\ochaseIS}[2]{\ochase(#1,#2)}
\newcommand{\ochasek}[3]{\ochase^{#3}(#1,#2)}
\newcommand{\ochaseshort}{\ochaseIS{\I}{\ruleset}}
\newcommand{\ochasekshort}[1]{\ochasek{\I}{\ruleset}{#1}}
\newcommand{\sochaseIS}[2]{\sochase(#1,#2)}
\newcommand{\sochasek}[3]{\sochase^{#3}(#1,#2)}
\newcommand{\sochaseshort}{\sochaseIS{\I}{\ruleset}}
\newcommand{\sochasekshort}[1]{\sochasek{\I}{\ruleset}{#1}}

\newcommand{\valdepth}[1]{
\mathsf{depth}_\exists(#1)}
\newcommand{\frdepth}[1]{
\mathsf{depth}^{\mathsf{fr}}_\exists(#1)}
\newcommand{\rank}[1]{\mathsf{rank}(#1)}
\renewcommand{\hom}{h}
\newcommand{\map}{\varphi}
\newcommand{\body}{B}
\newcommand{\head}{H}

\newcommand{\match}{\pi}

\newcommand{\safeo}{\match_{\mathsf{o}}}
\newcommand{\safeso}{\match_{\mathsf{so}}}
\newcommand{\safeany}{\match_{\star}}
\newcommand{\erule}{\sigma}
\newcommand{\ruleset}{\Sigma}
\newcommand{\fr}[1]{\mathsf{fr}(#1)}
\newcommand{\ex}[1]{\mathsf{ex}(#1)}

\newcommand{\freeze}[1]{\mathsf{freeze}(#1)}

\newcommand{\cto}{\mathsf{CT}^{\mathsf{o}}}
\newcommand{\ctso}{\mathsf{CT}^{\mathsf{so}}}
\newcommand{\ctany}{\mathsf{CT}^{\mathsf{\star}}}
\newcommand{\bno}{\mathsf{BN}^{\mathsf{o}}}
\newcommand{\bnso}{\mathsf{BN}^{\mathsf{so}}}
\newcommand{\bnany}{\mathsf{BN}^{\mathsf{\star}}}

\newcommand{\forw}{\mathsf{FO}\mbox{-}\mathsf{R}}

\newcommand{\forwaf}{\mathsf{FO}\mbox{-}\mathsf{R}^{\mathsf{AF}}}

\newcommand{\atom}{\alpha}
\newcommand{\query}{Q}
\newcommand{\const}{a}
\newcommand{\Const}{\mathcal{C}}
\newcommand{\Pred}{\mathcal{P}}
\newcommand{\Voc}{\mathcal{V}}

\renewcommand{\t}{v}
\newcommand{\kd}{k_{d}}
\newcommand{\kfo}{k_\mathrm{FO}}
\newcommand{\kaf}{k_\mathrm{AF}}

\newcommand{\forwany}{\mathsf{FO}\mbox{-}\mathsf{R}^{\mathcal{C}}}

\newcommand{\DF}[1]{\mathsf{DF}(#1)}
\newcommand{\foquery}{\phi}

\title{Oblivious and Semi-Oblivious Boundedness for Existential Rules}

\author{
Pierre~Bourhis$^{1,3,4}$\and
Michel~Leclère$^{2,4}$\and 
Marie-Laure~Mugnier$^{2,4}$\and
\\
Sophie~Tison$^{3,4}$\and
Federico~Ulliana$^{2,4}$\And
Lily Galois$^{3,4}$
\affiliations
$^1$ CNRS, France
\\$^{2}$ Univ. Montpellier, LIRMM, France
\\$^3$ Univ. Lille, CRIStAL, France
\\$^{4}$ Inria, France
\emails
\{firstname.lastname\}@inria.fr
}

\begin{document}

\maketitle

\begin{abstract}
We study the notion of boundedness in the context of positive existential rules, that is, whether there exists an upper bound to the depth of the chase procedure, that is independent from the initial instance. By focussing our attention on the oblivious and the semi-oblivious chase variants, we give a characterization of boundedness in terms of FO-rewritability and chase termination. We show that it is decidable to recognize if a set of rules is bounded for several classes and outline the complexity of the problem. 

\medskip
\emph{This report contains the paper published at IJCAI 2019 \emph{\cite{DBLP:conf/ijcai/BourhisLMTUG19}}  and an appendix with full proofs.}

\end{abstract}

\section{Introduction}
We consider the setting of ontology-based query answering (OBQA) in which answers to conjunctive queries are logically entailed 
from a knowledge base constituted of a set of facts (or database instance) and an ontology. 
Existential rules, also known as Tuple Generating Dependencies (TGDs) in database theory, 
are an expressive knowledge 
representation language well studied in the OBQA setting \cite{DBLP:books/sp/virgilio09/CaliGL09,DBLP:journals/ai/BagetLMS11,DBLP:journals/jair/CaliGK13}. 
These rules generalize function-free Horn rules (like those of datalog)
with existentially quantified variables in the rule heads, 
which allow one to assert the existence of unknown individuals, 
and hence to reason in open domains. 
Beside datalog, existential rules generalize the Semantic Web language RDF Schema, as well as most Description Logics used in the OBQA context, namely Horn description logics, in particular those at the core of the tractable profiles of the ontological language OWL 2.

The two main approaches developed to answer conjunctive queries on existential rules knowledge bases are \emph{materialization} and \emph{query rewriting}. 
Both can be seen as ways of reducing query answering to a classical database query evaluation problem.  
Materialization relies on a forward chaining technique, called the \emph{chase}, that consists in expanding the database instance with the facts entailed by rules until fixpoint. 
 In contrast, 
 query rewriting is a backward chaining mechanism 
 that consists in rewriting an input query using relevant rules,
  so that its answers on the knowledge base are exactly the answers of the rewritten query on the database instance alone. 
Query answering being undecidable for existential rules, both materialization and query rewriting may not terminate. 


This led to intensive research aiming at characterizing decidable and tractable classes of existential rules. Several syntactic restrictions were proposed to ensure chase termination (e.g. weak-acyclicity \cite{FaginKMP05}) or the existence of a (finite) first-order rewriting of a  conjunctive query, a property referred as \emph{FO-rewritability} \cite{DBLP:journals/jar/CalvaneseGLLR07}.
Nevertheless, the interactions between chase termination and FO-rewritability have been little investigated so far,
and not much is known for existential rules on which both hold. 
What are the relationships between these two properties?

Answering this question leads us to another fundamental problem, which
has been extensively studied for datalog, namely \emph{(uniform) boundedness} \cite{hillebrand1995undecidable}. 
Boundedness concerns the recursivity of rules, and asks whether there is an upper bound on the depth of the chase, which is independent from any database instance.
The property is key for practical optimization of reasoning as it implies that the ruleset is essentially non-recursive (although syntactic conditions may fail to capture this). 
It is known that boundedness and FO-rewritability are equivalent in the case of datalog \cite{ajtai1994datalog}, but this does not hold for existential rules. 
In this setting, the notion of boundedness also depends on the chase variant as they all behave differently with respect to termination.

We focus our attention on 
the \emph{oblivious} and \emph{semi-oblivious} (a.k.a. Skolem)
chase  \cite{DBLP:conf/pods/Marnette09}. 
As a matter of fact, almost all known sufficient conditions for chase termination fall within these chase variants
(from the simplest ones: rich-acyclicity \cite{DBLP:conf/pods/HernichS07}, weak-acyclicity \cite{FaginKMP05}  and  acyclic-GRD \cite{DBLP:journals/ai/BagetLMS11} to the more general MFA \cite{DBLP:journals/jair/GrauHKKMMW13}), at the exception of the recent work of \cite{DBLP:conf/ijcai/CarralDK17} which applies to the restricted chase variant. 
Importantly, we consider a \emph{breadth-first} version for both variants, which ensures the minimal depth of the chase \cite{RR18}.\footnote{See \cite{DBLP:journals/corr/abs-2004-10030} for an extended version of this conference paper, to appear in Theory and Practice of Logic Programming (added note  w.r.t. IJCAI 2019 paper).}

Our main contribution is 
a characterization of boundedness
in terms of  chase termination \emph{and} FO-rewritability.
This means that a set of rules is bounded if and only if it ensures both chase termination for any instance and FO-rewritability  for any conjunctive query. 
We show this by proving two orthogonal results.
The first is a bound on the depth of existential variables when the chase terminates on all instances. 
The second is a bound on the (breadth-first) rank at which facts using terms of a given depth are inferred. 

This connection reveals important differences between the two variants.
For the oblivious case we show that, when chase termination holds,  FO-rewritability on \emph{full-atomic queries} (queries with a single atom and only answer variables) is equivalent with FO-rewritability.  
Moreover, for the case of \emph{fully-existential rules} (rules where all head atoms have at least one existential variable), we show that chase termination is equivalent to boundedness and so it implies FO-rewritability. None of these properties hold for the semi-oblivious chase. 


Recognizing if a set of existential rules is bounded is undecidable already for datalog \cite{hillebrand1995undecidable}.
However, we show the decidability of the problem 
for major classes of existential rules as direct corollaries of our characterizations and existing results from the literature. Precisely, the problem is PSpace-complete for linear and sticky rules and in 2Exptime for guarded rules. Finally, we consider the \emph{$k$-boundedness} problem (i.e., whether the chase terminates in $k$ steps on all instances), which was recently proven decidable for several chase variants, including those investigated here  \cite{RR18}.  We show that deciding if a ruleset is $k$-bounded is in 2Exptime for the breadth-first {(semi-)} oblivious chase 
and co-NExptime-complete for datalog. 

\medskip
\emph{Proofs omitted due to space limitations are detailed in the appendix. }

\section{Preliminary Definitions}
We consider a relational vocabulary $\Voc=(\Pred, \Const)$ constituted of a finite set of predicates $\Pred$ and a finite set of constants $\Const$. 
A term $\t$ is a constant of $\Const$ or a variable. 
An atom is of the form $p(\t_1 \dots \t_k)$ where $p$ is a predicate of arity $k$ and the $\t_i$ are terms.
We denote by $\adom{}$ the set of its terms  and extend the notation to sets of atoms. 
An \emph{embedding} 
$\map$ from a set of atoms $A$ to a set of atoms $A'$ is a substitution of $\adom{A}$ with $\adom{A'}$ such that 
$\map(A)\subseteq A'$. A \emph{homomorphism} is an embedding which is the identity on constants.

An \emph{instance} $\I$ is a conjunction of atoms on constants and (globally) existentially quantified variables.
It is finite unless otherwise specified.
Throughout this paper, we see an instance $\I$ as the set of its atoms and call \emph{fact} any atom $f$ that belongs to this set.
Given a finite set $\Pred$ of predicates, the \emph{critical instance} $\Ia$ is composed of all facts built on $\Pred $ and special constant $a$. 
Any instance $\I$ on $\Pred$ can be embedded into $\Ia$.

An \emph{existential rule} $\erule $ is a closed formula $\forall\bar{x} \bar{y}(\body[\bar{x},\bar{y}]\rightarrow\exists \bar{z}\head[\bar{x},\bar{z}]$
where $\body$ and $\head$ are sets of atoms built on variables called the \emph{body} and the \emph{head} of the rule, also denoted by $\mathsf{body}(\erule)$ and $\mathsf{head}(\erule)$ respectively. 
The set of variables $\bar x$ shared by $\body$ and $\head$ is called the \emph{frontier} of the rule and is denoted by $\fr{\erule}$.
The set of variables $\bar z$ that  belong to $\head$ only are called \emph{existential variables} and are denoted by $\ex{\erule}$. 
Universal quantifiers will often be omitted in the remainder of the paper.
A rule such that $\ex{\erule}=\emptyset$ is called \emph{datalog}.
A rule where all head atoms contain at least one existential variable is called \emph{fully-existential} and denoted by \emph{FE-rule}.
We say that a rule $\erule$ is applicable on an instance $\I$ if there is a homomorphism $\match$ from $\mathsf{body}(\erule)$ to $\I$ and call the pair  $(\erule,\match)$ a \emph{trigger} of $\I$.
Given a trigger $(\erule,\match)$, we denote by ${\match}_{|\mathsf{fr}(\erule)}\subseteq \match$ the restriction of $\match$ to  $\fr{\erule}$.

A \emph{knowledge base} (KB) is a pair $(\I,\ruleset)$ where $\I$ is an instance and $\ruleset$ a set of existential rules.
The chase is a fundamental tool for computing logical consequences from a KB since, when it terminates, it computes a universal model of the KB, i.e., a model that maps by homomorphism to any other model of the KB (with a model being seen here as an instance). 
In this work, we focus our attention on the breadth-first oblivious ($\ochase$) and semi-oblivious ($\sochase$) variants. As discussed in Section
\ref{subsection-bf}, the breadth-first behavior is particularly interesting when studying boundedness.


\begin{definition}
Let $(\I,\ruleset)$ be a knowledge base and  $\star\in\{\mathsf{o},\mathsf{so}\}$ a chase variant.
Then, the \emph{breadth-first} $\anychase$ is defined as follows: $\anychasek{\I}{\ruleset}{0}=I$ and for all saturation rank $i\geq0$
$$
\anychasekshort{i+1}=
\anychasekshort{i}\;\cup \bigcup_{(\erule,\match)} \safeany(\mathsf{head}(\erule))
$$
where $(\erule,\match)$ is any trigger of $\anychasekshort{i}$
and $\safeany\supseteq \match$  a substitution  that replaces each existential variable $z\in\ex{\erule}$ with a fresh variable named as follows: 
\begin{itemize}
\item $\safeo(z)= z_{(\erule,\match)}$ 
\item  $\safeso(z)= z_{(\erule,{\match}_{|\mathsf{fr}(\erule)})}$
\end{itemize}
Then, we define $\anychaseshort=\bigcup^{\infty}_{i\geq0}\anychasekshort{i}$.
The $\star$-chase \emph{terminates} on $(\I,\ruleset)$ if there is a rank $k$  
with $\anychaseshort$ = $\anychasekshort{k}$. 

\end{definition}

Note that for the $\ochase$ fresh variables are named by the trigger from which they have been generated. 
Instead, for the $\sochase$ the naming only depends on the frontier-restriction of the homomorphism of the trigger.
This means that any two triggers having the same rule and agreeing on the image of its frontier variables produce \emph{equal} results,
hence only one of them is actually considered by the $\sochase$.
The $\sochase$ is very close to the \emph{Skolem chase}, which relies on a skolemisation of the rules: first, each rule
$\erule$ is transformed by replacing each occurrence of an existential variable $z$ with a
functional term $f^\erule_z(\fr{\erule})$ on the frontier of $\erule$; then the $\ochase$ is run on the skolemised rules. 
At each saturation rank, the Skolem chase produces a result isomorphic to that of the $\sochase$ (up to the renaming of each
Skolem term by the corresponding fresh variable), hence the forthcoming results on the $\sochase$ also hold for the Skolem chase. 

\begin{example} \label{ex-chase} Consider the rule $
\erule = p(x,y) \rightarrow \exists z~ p(x,z)$. 
Then $\ochaseshort$ with $\I = \{p(a,b)\}$ and $\ruleset=\{\erule\}$ is infinite - as the chase does not terminate. The atom $p(a,z_{(\erule,\match_1)})$ with $\match_1 = \{x\mapsto a, y \mapsto b\}$ is first inferred, then  
$p(a,z_{(\erule,\match_2)})$ with $\match_2 = \{x\mapsto a, y \mapsto z_{(\erule,\match_1)}\}$, and so on. Here, each rule application enables a new trigger.
In contrast, $\sochaseshort$ is finite, in that only the first rule application will be performed, producing  $p(a,z_{(\erule,\{x\mapsto a\})})$, since all triggers map the frontier variable $x$ to $a$. For the Skolem chase,  $\erule$ is rewritten as $\erule' = p(x,y) \rightarrow p(x,f^\erule_z(x))$. The first rule application according to trigger $(\erule', \match_1)$  produces $p(a,f^\erule_z(a))$, then the chase halts as the same atom is produced by the next trigger.   
\end{example}

\begin{definition}
\label{def:rank}
The \emph{rank} of a fact $f\in\anychaseshort$, 
denoted by $\rank{f}$, is $0$ if $f\in\I$ and $1+\max\{\rank {f'} | f' \in \match(\mathsf{body}(\erule))\}$
 if $f$ is produced by the trigger $(\erule, \match)$. 
This definition is naturally extended to terms and sets of facts.
The \emph{rank} of $\anychaseshort$ is the smallest $k$ such that $\anychaseshort$ = $\anychasekshort{k}$ if $\anychaseshort$ terminates, and it is infinite otherwise. 
\end{definition}
Note that for the breadth-first chases we consider the above definition implies that $\rank{f}$ is the smallest $k$ such that $f\in\anychasekshort{k}\setminus\anychasekshort{k-1}$.

An \emph{FO-query} $\foquery(x_1,...,x_n)$ is a (function free) first-order formula whose free variables (called answer variables) are exactly $\{x_1,...,x_n\}$.
A \emph{conjunctive query} (CQ) is an FO-query which is an existentially quantified conjunction of atoms. 
An \emph{atomic} query is a CQ with a single atom. A \emph{full-atomic} query is an atomic query where all terms are free variables.  
A query is called  \emph{Boolean} if it does not have any free variable. 
%
As for instances, it will be handful to see CQs as sets of atoms, of course by distinguishing the answer variables.
A \emph{union of conjunctive queries} (UCQ) $\mathcal \query$ is a disjunction of CQs with the same free variables, also seen as a set of CQs. 

A tuple of constants $(\const_1, ..., \const_n)\in\Const^n$  is an \emph{answer} to a CQ $\query(x_1,...,x_n)$ on an instance $\I$  if there is a homomorphism $\hom$ from $\query$ to $\I$ such that $\hom(x_i) = a_i$ for $1 \leq i \leq n$. 
Equivalently, $\I\models\query[x_i \mapsto a_i]$, where $\models$ denotes the classical logical consequence and
$\query[x_i \mapsto a_i]$ is the Boolean query obtained from $\query$ substituting each $x_i $ with $a_i$.
A tuple of constants $(\const_1, ..., \const_n)\in\Const^n$  is a \emph{certain answer} to $\query$ on a KB $(\I, \ruleset)$ if $\I, \ruleset \models Q[x_i \mapsto a_i]$.  This is equivalent to the existence of a saturation rank  $k$ such that $\anychasekshort{k}\models\query[x_i \mapsto a_i]$. In other words,  the certain answers to $Q$ on $(\I, \ruleset)$ are exactly its answers on the possibly infinite instance $\anychaseshort$. The set of (certain) answers to a UCQ $\mathcal \query$ is the union of the sets of (certain) answers to the CQs it contains.

\subsection{Termination vs Boundedness}
To begin our study, we need to present the relationships between chase termination and boundedness. 
Let $\star\in\{\mathsf{o},\mathsf{so}\}$ be a chase variant, the $\star$-chase termination class, denoted by $\ctany$, contains all rulesets $\ruleset$ such that $\anychaseshort$ terminates for all instances $\I$. The $\star$-boundedness class, denoted by $\bnany$, contains all bounded rulesets $\ruleset$, i.e., for which there exists an integer $k$ such that 
$\anychasekshort{k}=\anychaseshort$ for all instances $\I$ . Obviously, $\bnany \subset \ctany$. 

\begin{example}\label{ex-bounded-termination} Let $\erule_1 = p(x,y) \land p(y,z) \rightarrow p(x,z)$ and $\erule_2 = p(x,y) \land p(w,z) \rightarrow p(x,z)$. 
Because both rules are datalog, $\{\erule_1\} \in \ctany$ and $\{\erule_2\} \in \ctany$.  
However, $\ruleset=\{\erule_1\} \not \in \bnany$, since the rank of $\anychaseshort$ depends on $\I$. In contrast, $\{\erule_2\} \in \bnany$ and the bound is $k = 1$. Similarly, 
$\{\sigma_1,\sigma_2\} \in \bnany$. Indeed, $\erule_2$ produces at the first rank all atoms that can be produced by $\erule_1$ at later ranks.
\end{example}

To get a better understanding of boundedness, it will be useful 
to decompose each rule of a set thereby distinguishing between its ``datalog part'' and its ``existential part''.  For instance, a rule of the form $p(x,y) \rightarrow \exists z~p(x,z) \land q(x)$ can be decomposed into a datalog rule $p(x,y) \rightarrow q(x)$ and an FE-rule $p(x,y) \rightarrow \exists z~p(x,z)$. 
Let $\erule$ be any existential rule
  of the form $\body \rightarrow \head_F \land \head_D$
where $B$ is the set of body atoms, $\head_F$ is the set of head atoms with at least one existential variable and $\head_D$ are the remaining head atoms.
The datalog-fully existential decomposition of $\erule$, denoted by $\DF{\erule}$, returns a set made of the FE-rule 
$\body \rightarrow \head_F$ 
together with a (single head) datalog rule of the form $\body \rightarrow {\head}_D^i$, for each  $\head_D^i\in\head_D$. 
The definition is then extended to sets $\DF{\ruleset}=\bigcup_{\erule\in\ruleset}\DF\erule$.
This decomposition preserves 
boundedness and termination of the oblivious chase.\footnote{This is not true for the $\sochase$. For instance, for $\Sigma = \{ \sigma = p(x,y) \rightarrow \exists z ~p(x,z) \land q(x,y) \}$ and $I = \{p(a,b)\}$, $\sochase(\Sigma, I)$ is infinite, while $\DF{\Sigma}$ is $\mathsf{so}$-bounded. This is due to the fact that $\sigma$ has frontier $\{x,y\}$, while the FE-rule $p(x,y) \rightarrow \exists z ~p(x,z)$ in $\DF{\erule}$ has frontier $\{x\}$. We correct here a wrong claim in Proposition 1 of IJCAI's paper, which has no incidence on the paper's results.}

\begin{proposition} \label{prop-DF}
$\ruleset\in\cto$ iff $\DF\ruleset\in\cto$ and 
$\ruleset\in\bno$ iff $\DF\ruleset\in\bno$.
\end{proposition}


\section{Upper Bounding the Chase Depth}

Our approach consists of defining a notion of \emph{existential depth} for facts, proper to each chase, which is 
 finite on a given instance if and only if the chase terminates on that instance. Then we show that for each chase, 
the existential depth of all facts produced by the chase for a given ruleset are bounded by those of the critical instance. This means that whenever the chase terminates on the critical instance there is an upper bound to the existential depth of the facts, for all instances.
In the next section, with these results in hand, we use FO-rewritability to bound the \emph{rank} at which any fact of a certain existential depth will be inferred. This will give us a characterization of boundedness for the oblivious and $\sochase$ in terms of FO-rewritability and chase termination.

\subsection{The Oblivious Case}

Intuitively, the notion of existential depth of a term measures the number of fresh variable generation steps that led to the creation of this term.  

\begin{definition}
The \emph{existential depth} (or simply \emph{depth}) \emph{of a term}  $\t$ that belongs to $\ochaseshort $ is
$$\valdepth{v}=
\left\{
\begin{array}{lr}
0&\mbox{if } v\in \adom{\I}
\\
1+\max\{\;\valdepth{v_B}\;\}
&\mbox{otherwise} 
\end{array}
\right.
$$
where $v_B$ is any term in ${\match(\mathsf{body}(\erule))}$ used by 
 a trigger $(\erule,\match)$ which generates $\t$.
The \emph{existential depth} of a fact $f$ is the maximum existential depth of its terms.
The \emph{existential depth} of $\ochaseshort$ is the maximum existential depth of its facts if it  is finite  and is infinite otherwise.
\end{definition}

To illustrate the definition, consider Example \ref{ex-chase}.
The existential depth of terms in $\ochase(\I, \{\erule\})$ is unbounded, which is in line with the non-termination of the $\ochase$ 
on $(\I, \{\erule\})$. 
The rule
 $\sigma_1$ in Example \ref{ex-bounded-termination}.
 shows the difference between rank and existential depth. For any $\I$, the existential depth of terms (hence facts) is $0$ because $\sigma_1$ is datalog, however their rank depends on $\I$. More generally, for any term $\t$ and fact $f$ in $\ochaseshort $ it holds that $\valdepth{v}\leq\rank{v}$ and $\valdepth{f}\leq\rank{f}$. 
 Hence, if $\ochaseshort $ terminates,  its existential depth is finite. Reciprocally, when  the existential depth of  $\ochaseshort $ is finite, so it is the number of its terms, and $\ochaseshort$ terminates. 
%
We point out that  when dealing with sets of \emph{FE-rules} the notions of rank and existential depth coincide, as illustrated by Example \ref{ex-chase}.
\begin{proposition}
\label{prop-o-depth-rank}
If $\ruleset$ is a set of FE-rules then, for all instance $\I$ and term $\t$ in $\ochaseshort$,
 holds that $\valdepth{v}=\rank{v}$.
\end{proposition}

It should be clear that, for a given ruleset, the $\ochase$ may have unbounded rank even when it terminates on all instances (see for instance Example  \ref{ex-bounded-termination}).
Nevertheless, when a ruleset is in $\cto$, our goal is to show that there exists a bound on the \emph{existential depth} of its terms, which holds for all instances.
Aiming at this, we present a lemma stating that existential depth of terms are preserved by embeddings.

 \begin{lemma}\label{o-depth-preservation}
 \label{lemma-depth}
 For any  embedding  $\map$ from $\I$ to $\I'$ and any $i \geq 0$, there exists an embedding $\map'\supseteq \map$ from $\ochasekshort{i}$
   to $\ochasek{\I'}{\ruleset}{i}$ which preserves the existential depth of terms, i.e., for every term $\t$ in $\ochaseshort$ it holds that
$\valdepth{v}=\valdepth{\map'(v)}$. 
 \end{lemma}

It is well-known that the $\ochase$ terminates on all instances if and only if it terminates on the critical instance \cite{DBLP:conf/pods/Marnette09}. We leverage this property to compute a bound on the existential depth under chase termination.

\begin{theorem}
\label{thm-finitedepth-O}
When $\ruleset\in\cto$ there exists a constant $\kd$ such that for every instance $I$, the existential depth of a term in $\ochaseshort$ is bounded by $\kd$.
\end{theorem}

\begin{proof}
Because $\ruleset\in\cto$, the $\ochase$ terminates on the critical instance $\Ia$.
Let $\kd$ be the largest rank 
such that 
$\adom{\ochasek{\Ia}{\ruleset}{\kd}}\setminus\adom{\ochasek{\Ia}{\ruleset}{\kd-1}}\neq\emptyset$. 
Every instance $\I$ can be embedded into $\Ia$.
By Lemma \ref{lemma-depth} the existential depth of the terms in $\ochasek{\I}{\ruleset}{}$ is bounded by that of $\ochaseIS{\Ia}{\ruleset}$, which is in turn bounded by $\kd$.
\end{proof}

  Chase termination is a necessary condition for boundedness as it bounds the existential depths of the variables generated by the chase - but not the rank (see the datalog case).
Interestingly, for FE-rules, 
chase termination also becomes a sufficient condition for boundedness, because the notion of rank and existential depth coincide (Proposition \ref{prop-o-depth-rank}).

\begin{corollary}
\label{cor-CT-BN-fully}
For $\ruleset$ a set of FE-rules, $\ruleset\in\cto$ iff $\ruleset\in\bno$.
\end{corollary}

For general existential rules, we will  later show that when  a restricted form of FO-rewritability holds, one can also provide a bound to the rank of the $\ochase$ (Theorem \ref{thm-bn-foaf-cto}).

\subsection{The Semi-Oblivious Case}
When applied to the $\sochase$, the previous notion of existential depth is not preserved by embedding, which hinders the possibility of using the critical instance to bound the existential depth of terms.
As illustrated below, this is due to the fact that the $\sochase$ makes equal the result of two distinct triggers agreeing on a rule frontier. 
\begin{example}
Consider $\I = \{p(a,b)\}$, $\I' = \I\cup\{r(a,b)\}$ and $\Sigma = \{ \erule_1: p(x,y) \rightarrow \exists z ~r(z,y) \;\; \erule_2: r(x,y) \rightarrow \exists z~s(y,z)\}$.
Then, 
$\sochasekshort{2}=
\I
\cup\{\, r(z_{(\erule_1,\match)},b)
\cup s(b,z_{(\erule_2,\match)})\,\}
$
with $\match=\{y{\ \mapsto\ } b\}$.
Also, 
$\sochasekshort{2}\subseteq\sochasek{\I'}{\ruleset}{1}$
because all triggers applied by the chase from $\I$ are already applicable on $\I'$.
The application of $\erule_2$ on  
$r(a,b)$ and $r(z_{(\erule_1,\match)},b)$ gives equal results, hence $\sochasek{\I'}{\ruleset}{1}=\sochasek{\I'}{\ruleset}{2}$.
In the embeddings from $\sochasek{\I}{\ruleset}{2}$ to $\sochasek{\I'}{\ruleset}{2}$, $z_{(\erule_2,\match)}$ is mapped to itself,
but both occurrences have different existential depth (resp. 2 and 1).  
\end{example} 

It is therefore natural to turn to the following notion of depth, which accounts for frontier terms only. 

\begin{definition}
The \emph{frontier existential depth} (or simply frontier depth) of a term $\t$ that belongs to $\sochaseshort$ is
$$\frdepth{v}=
\left\{
\begin{array}{lr}
0&\mbox{if } v\in \adom{\I}
\\[0.9mm]
1&\mbox{if } \fr\erule=\emptyset
\\
1+\max\{\;\frdepth{v_B}\;\}
&\mbox{otherwise} 
\end{array}
\right.
$$
where $v_B$ is any term in ${\match(\fr\erule)}$ used by 
 a trigger $(\erule,\match)$ which generates $\t$.
Accordingly, the \emph{frontier depth} of a fact $f$ is the maximum frontier depth of its terms.
The frontier depth of $\sochaseshort$ is defined as the maximum frontier depth of its facts if it  is finite  and is infinite otherwise.

\end{definition}

Note that frontier depth coincides with the (usual) depth of terms generated by the Skolem chase. 

Clearly, $\frdepth\t\leq\valdepth\t$ for all $\t$ in $\ochaseshort$.  The following example illustrates the difference between the two notions of (existential) depth.

\begin{example}\label{ex-frdepth} Let $\ruleset = \{\erule = p(x,y,u)\rightarrow \exists z~p(y,x,z)\}$. Starting from $\I = \{p(a,b,c)\}$, the o-chase generates an infinite number of fresh variables $v$ with increasing $\valdepth{v}$. The rank of the so-chase is instead 2 and for each fresh variable $v$, $\frdepth{v} = 1$ as all triggers map $\fr{\erule}$ to $\adom{\I}$. 
\end{example}

It is worth noting that not only the oblivious notion of exitential depth is not effective for studying the $\sochase$, but also that the frontier depth is not well characterizing the behavior of the $\ochase$ either. The crux is that
the finiteness of the frontier depth cannot be related with the termination of the $\ochase$, as illustrated by Example \ref{ex-frdepth}. 
Using such a notion to study the $\ochase$ would impede us, for instance, to establish Corollary \ref{cor-CT-BN-fully}, which relies on the fact that rank and existential depth coincide for the oblivious-chase (Property \ref{prop-o-depth-rank}).



\smallskip
We are now ready to show that the frontier depth is preserved by embeddings.
The next lemma and theorem are the counter-parts of Lemma \ref{o-depth-preservation} and Theorem \ref{thm-finitedepth-O} for the $\sochase$. 

 \begin{lemma}
 \label{so-depth-preservation}
 For any embedding   $\map$ from $\I$ to $\I'$ and any $i \geq 0$,  there exists an embedding $\map'\supseteq \map$ from $\sochasekshort{i}$
   to $\sochasek{\I'}{\ruleset}{i}$ which preserves the frontier depth of terms.
\end{lemma}

\begin{theorem}
\label{thm-finitedepth-SO}
When $\ruleset\in\ctso$ there exists a constant $\kd$ such that for every instance $I$, the frontier depth of a term in $\sochaseshort$ is bounded by $\kd$.
\end{theorem}



\subsection{On the Interest of the Breadth-First Chase}\label{subsection-bf}

We conclude this section with some remarks on the interest of studying boundedness for breadth-first chases. We assume that the reader is familiar with the notion of chase sequence.\footnote{A chase sequence is any sequence of triggers satisfying the applicability criterion of the chase variant. For the oblivious chase, the same trigger should not be applied twice. For the semi-oblivious chase a trigger is not applied if a trigger for the same rule assigning the same image for the frontier variables has been applied before.}
We define the rank of a chase sequence on $(\I, \ruleset)$ as the maximal rank of its facts if it is finite, and infinite otherwise. 

 For the (semi-)oblivious chase, it is well-known that there is a terminating chase sequence for $(\I, \ruleset)$ if and only if all chase sequences for $(\I, \ruleset)$ terminate. However, not all terminating chase sequences have the same rank, and the minimal rank is obtained with breadth-first sequences \cite{RR18}.
This makes the notion of boundedness we consider equivalent to studying whether there exists a bound such that, for all instance, \emph{there exists} a terminating chasing sequence whose rank is within the bound. Hence, 
it characterizes the fact that the chase can indeed terminate within that bound, if a strategy ensuring a minimal sequence rank is followed. 
It is therefore natural to consider breadth-first chases which achieve this property, like the (semi-)oblivious chase. 
Example \ref{ex-bounded-termination} illustrates this concept and shows that, already for datalog, the rank of some chase sequences may be not bounded, while the rank of all breadth-first  sequences is bounded. This happens for instance if all applications of the transitivity rule $\erule_2$ are performed before the rule  $\erule_1$.

In the special case of FE-rules, it is not hard to see that all oblivious chase sequences for $(\I, \ruleset)$ have the same rank. However, this does not hold for the semi-oblivious chase. Below, a variation of 
Example \ref{ex-bounded-termination}, where some dummy variables are introduced, illustrates this point.

\begin{example}\label{ex-bounded-termination-2} Let $\ruleset = \{\erule_1, \erule_2\}$, with $\erule_1 = p(x,y,t) \land p(y,z,u) \rightarrow \exists v~ p(x,z,v)$ and $\erule_2 = p(x,y,t) \land p(w,z,u) \rightarrow \exists v~ p(x,z,v)$. 
The rank of $\sochaseshort$ is bounded by 2 for any $\I$, while again  performing all applications of $\erule_2$  before $\erule_1$ gives derivations of different ranks. 
\end{example}

\section{The Impact of First Order Rewritability}
We now turn our attention to FO-rewritability and show that it yields a bound on the rank of specific (sets of) facts that share terms with the initial instance $\I$. For the $\ochase$, we bound the rank of facts that have all their terms in $\I$. 
For the $\sochase$, we consider triggers that map a rule frontier to terms of $\I$: we do not bound the rank of facts that allow to fire such triggers, but we show that for each such trigger $t = (\erule,\match)$, there is a trigger $t' = (\erule,\match')$ that agrees with $t$ on the mapping of $\fr{\erule}$ and that is fired at a bounded rank. 
In Section \ref{sec-bound}, we will leverage these results to show that FO-rewritability yields a bound on the rank of all facts with a certain existential depth. For the $\ochase$, a restricted version of FO-rewritability is sufficient to get these properties.


We say that a pair $(Q, \ruleset)$ is \emph{FO-rewritable} (resp. UCQ-rewritable) if there is an FO-query  (resp. a UCQ) $\mathcal \query$  such that, for all $\I$, the certain answers to $\query$ on $(\I, \ruleset)$ are exactly the answers to  $\mathcal \query$ on $\I$. It is known that FO-rewritability is equivalent to UCQ-rewritability.\footnote{It follows from the (Finite) Homomorphism preservation theorem, a classical result in model theory \cite{DBLP:journals/jacm/Rossman08}.} A set of rules $\ruleset$ is \emph{FO-rewritable} (or equivalently, UCQ-rewritable) if $(Q, \ruleset)$ is FO-rewritable for every CQ $Q$. We denote by $\forw$ the class of FO-rewritable rulesets.  We will also consider specific classes of CQs. Given a class of CQs $\mathcal C$,
we say that a ruleset $\ruleset$ is FO-rewritable with respect to $\mathcal C$ if  $(Q, \ruleset)$ is FO-rewritable for all $Q \in \mathcal C$.
We denote by $\forwany$ the corresponding class. We first point out that FO-rewritability with respect to full-atomic queries, denoted by $\forwaf$, is a strictly weaker property than FO-rewritability.

\begin{proposition}
\label{prop:foafsupset}
$\forwaf \supset\forw$
\end{proposition}

\begin{proof}
The inclusion holds by definition, and to see that it is strict 
 consider $\ruleset=\{ \erule = p(x,x_1),p(x_1,x_2),p(x_2,z)\rightarrow \exists y ~p(x,y),p(y,z)\}$. $\ruleset$ is not FO-rewritable as for the Boolean query $\query=\{p(a,u), p(u,b)\}$, where $a$ and $b$ are constants, $(Q, \ruleset)$ is not FO-rewritable (we would need an infinite union of Boolean CQs of the form 
 $\{ p(a,u_0), ...p(u_{i-1},u_i), p(u_i, b)\}$, none of these queries being contained in another).
However, $\ruleset \in \forwaf$ as $(Q,\ruleset)$ is FO-rewritable for any $Q \in AF$. Indeed, $\erule$ cannot bring any answer to such query (in more technical terms, an existential variable of $\erule$ cannot be unified with an answer variable). 
\end{proof}
Note also that since full-atomic queries have only answer variables, they cannot be rewritten by means of \textit{FE-rules}. Thus, every set of FE-rules is trivially in $\forwaf$.
More interestingly, 
to check if $\ruleset\in\forwaf$
one can restrict the full-atomic queries of interest  
to those corresponding to the heads of the datalog rules yielded by the $\mathsf{DF}$-decomposition of $\ruleset$.  

\begin{proposition}\label{prop-FO-atomicfull-instance-datalogheads}
Let $\ruleset$ be a ruleset and $\mathsf{HD}_\ruleset$ be the full-atomic queries given by heads of the datalog rules in $\DF\ruleset$.
Then,
 $\ruleset\in \forwaf $ if and only if
  $\ruleset\in \forw^{\mathsf{HD}_\ruleset}$.
\end{proposition}


The following lemma upper bounds the rank of all facts with terms in $\I$ for sets of rules enjoying  FO-rewritability on full-atomic queries.

\begin{lemma}
\label{lem-FORWAF-rank}
 If $\ruleset\in\forwaf$ there is a constant $\kaf$ such that, for any instance $\I$ and fact $f$ such that $\adom f \subseteq \adom\I$, when $f \in \ochaseshort$ it holds that  $\rank{f} \leq \kaf$.
\end{lemma}

\begin{proof} The number of (non-isomorphic) full-atomic queries to be considered is finite,  as for Proposition \ref{prop-FO-atomicfull-instance-datalogheads}.
We take for $\kaf$ the maximal number of breadth-first rewriting steps necessary to obtain a UCQ-rewriting of a full-atomic query (we refer here to the breadth-first rewriting based on aggregated piece-unifiers, see \cite{DBLP:conf/rr/KonigLMT13}).
\end{proof}

The previous lemma also holds for the $\sochase$, however  we want to derive a bound on the rank of facts with a certain frontier depth, and for that full-atomic rewritability is not enough. To illustrate, consider $\ruleset = \{ \erule = p(x,y,u), p(y,z,v) \rightarrow \exists w~p(x,z,w) \}$. Here $\ruleset \in \forwaf$ (the only rewriting of a full-atomic query is the query itself because of the existential variable $w$). For any instance $\I$, the frontier depth of facts in the $\sochase$ is bounded by 1, however there is no bound on their rank (although the $\sochase$ terminates). Therefore, we give a different property for the $\sochase$, which requires the power of FO-rewritability.


\begin{lemma}
\label{lem-FORW-rank}
If $\ruleset\in\forw$ there is a constant $\kfo$ such that, for any instance $\I$ and any trigger
 $(\erule,\match)$ from $\sochaseshort$
  with $\match(\fr{\erule}) \subseteq \adom{\I}$,
there is also a trigger $(\erule,\match')$ from 
$\sochaseshort $ such that $\match'_{|\fr{\erule}}=\ \match_{|\fr{\erule}}$ and  $\rank{f} \leq \kfo$ for all $f\in\match'(\mathsf{body}(\erule))$.
\end{lemma}

%


\begin{proof} Similar to the proof of Lemma  \ref{lem-FORWAF-rank} but considering CQs of the form $\query_{\mathsf{body}(\erule)}$ whose atoms correspond to the atoms of $\mathsf{body}(\erule)$, for $\erule\in\ruleset$, and all variables are existentially quantified except for those in $\fr\erule$.
The number of such queries is bounded by the cardinal of $\ruleset$. We take for $\kfo$ the maximal number of breadth-first rewriting steps necessary to obtain a UCQ-rewriting from any $\query_{\mathsf{body}(\erule)}$ query. The proof actually shows that FO-rewritability with respect to rule body queries  is sufficient to derive the lemma.  
\end{proof}

\section{Boundedness: Linking Depth and Rank}\label{sec-bound}
We can finally establish a connection between the rank and depth of a fact 
when the chase is run on FO-rewritable sets of rules.
This will immediately lead us to a characterization of boundedness for the oblivious and semi-oblivious chases.

%

 \begin{theorem}
 \label{thm-FO-depth-O}
If $\ruleset\in\forwaf$ then 
 for all instance $\I$ and fact $f\in\ochaseshort $ 
 we have that 
 $\rank{f}\leq \valdepth{f}\times (\kaf + 1)+\kaf $
 with 
 $\kaf$ the bound provided by Lemma \ref{lem-FORWAF-rank}.
 \end{theorem}
 

\begin{theorem}
 \label{thm-FO-depth-SO}
If $\ruleset\in\forw$
 then 
 for all instance $\I$ and fact $f\in\sochaseshort $ 
  we have that 
 $\rank{f}\leq \frdepth{f}\times (\kfo + 1)+\kfo $
 with 
 $\kfo$  the bound provided by Lemma \ref{lem-FORW-rank}.
 \end{theorem}
For the $\ochase$, boundedness is exactly termination and FO-rewritability on full-atomic queries. Furthermore, for rulesets in $\cto$, the notions of $\forw$ and $\forwaf$ coincide.
\begin{theorem}
$\forwaf \cap \cto = \bno = \forw \cap \cto $\label{thm-bn-foaf-cto}
\end{theorem}
\begin{proof}
We start by showing  that $\bno \subseteq \forw \cap \cto $. By definition $\bno\subseteq\cto$. Then, $\bno \subseteq \forw $ follows from the equivalence between $\forw$ and the bounded-depth derivation property \cite{gottlob2014price}. Moreover, by Proposition \ref{prop:foafsupset} we have $\bno \subseteq \forwaf \cap \cto $.
To conclude the proof, 
by Theorem \ref{thm-finitedepth-O}  and  \ref{thm-FO-depth-O} we have that 
$ \forwaf \cap \cto \subseteq \bno$
and again by Proposition \ref{prop:foafsupset} follows $ \forw \cap \cto \subseteq \bno$.
\end{proof}
%
%

For the $\sochase$, boundedness can be characterized again as termination and FO-rewritability by Theorem \ref{thm-finitedepth-SO}  and  \ref{thm-FO-depth-SO}.
\begin{theorem}
\label{thm-bn-fo-ctso}
$\bnso=\forw\cap \ctso $
\end{theorem}


Summing up, we have the following differences between boundedness for $\ochase$ and $\sochase$. 
$\ochase$-boundedness requires $i)$ $\ochase$ termination and full-atomic-rewritability and  $ii)$ is equivalent to $\ochase$ termination for FE-rules. 
Intuitively, when a set of rules $\ruleset$ is decomposed into $\DF\ruleset$, the fully-existential part may cause non-termination of the $\ochase$, while the datalog part may cause non-FO-rewritability. Furthermore, the fully-atomic queries possibly leading to infinite rewritings in this case correspond to the heads of the datalog rules. Note however that this restricted form of FO-rewritability has still to be verified with respect to the whole set of rules. 
In contrast, $\sochase$-boundedness $i)$ requires a stronger form of FO-rewritability and $ii)$ FE-rules do not behave differently from general existential rules for this chase. Intuitively, for the $\sochase$, any existential rule (even an FE-rule) has an ``underlying'' datalog rule. 
This is illustrated by the following transformation. To each rule $\erule$ in $\ruleset$ 
we assign a special predicate $p_{\erule}$ of arity $|\fr{\erule}|$. $\Psi(\ruleset)$ is obtained from $\ruleset$ by replacing each rule $\erule = B \rightarrow H$ with two rules: a datalog rule $B \rightarrow p_{\erule}(\fr{\erule})$ and a rule $p_{\erule}(\fr{\erule}) \rightarrow H$. 
It can be shown that 
$\ruleset \in \ctso$ iff $\Psi(\ruleset) \in \cto$ 
and that 
$\ruleset \in \bnso$ iff $\Psi(\ruleset) \in \bno$. 
This may also provide an alternative path to study $\sochase$ boundedness by reducing it to $\ochase$ boundedness.






\section{Decidability and Complexity}

From the undecidability of (uniform) boundedness of datalog \cite{hillebrand1995undecidable}, we immediately obtain the undecidability of membership to $\bno$ and $\bnso$. 
A notable class of datalog rules with decidable boundedness (more precisely in linear time) is chain datalog \cite{GuessarianV94}. 
We obtain that membership to $\bnso$, $\ctso$ and $\forw$ remains undecidable for \textit{FE-rules}, while the decidability of membership to $\bno$, hence to $\cto$, is still open.\footnote{See Proposition \ref{prop-datalog-fully-exist} in the Appendix.}

Importantly, new decidability and complexity results about boundedness for specific existential rules studied in the literature can be obtained as direct corollaries of our results. This is in particular the case for classes known to be FO-rewritable.

\begin{corollary} 
\label{cor-decidability}
For any class of existential rules $\mathcal C \in \forw$, it holds that: $\mathcal C \in \bno$ iff $\mathcal C \in \cto$, and $\mathcal C \in \bnso$ iff $\mathcal C \in \ctso$.
\end{corollary}

This implies that membership to $\bno$ and $\bnso$ is PSpace-complete for the two main classes of FO-rewritable existential rules, namely \emph{linear} and \emph{sticky}. Indeed, deciding  $\cto$ and $\ctso$ is PSpace-complete for both \cite{DBLP:conf/pods/CalauttiGP15,CalauttiP19}. 
We also get an upper bound on the complexity of membership to $\bno$ and $\bnso$ for a major class of existential rules, namely \emph{guarded}.
This class is neither 
 $\ctso$ nor $\forw$. 
 However, membership to $\cto$ and $\ctso$
 for guarded rules is decidable in 2Exptime
  \cite{DBLP:conf/pods/CalauttiGP15}.
Then a careful reduction from \cite{DBLP:conf/ijcai/BarceloBLP18} allows us to set the result. The paper shows that checking FO-rewritability for a single query under guarded rules is in 2Exptime. This suffices since by Lemma \ref{lem-FORWAF-rank} and \ref{lem-FORW-rank}  we need to test only a polynomial number of queries.
%
%

We conclude by considering the  $k$-boundedness problem, which asks whether the chase actually halts within $k$ steps.
The problem is decidable for the breadth-first (semi-)oblivious chase and any set of existential rules \cite{RR18}.
Therefore, the $k$-boundedness question becomes interesting for dealing with fragments of existential rules where boundedness is undecidable. 
We study here the complexity of the following version of the problem.
Given 
a ruleset $\Sigma$ and a (unary encoded) integer $k$, does it hold that $\anychasekshort{k}=\anychaseshort$  for all instance~$\I$?

\begin{theorem} \label{th-k-bounded} Deciding \emph{$k$-boundedness} is in 2Exptime for existential rules for the $\ochase$ and $\sochase$; co-NExptime-complete for datalog rules; in co-NExptime~on~FE-rules~for~the~$\ochase$.
\end{theorem}

\begin{proof}(Sketch) The upper bound results rely on the decidability arguments from \cite{RR18}. 
Co-NExptime-hardness for datalog is by reduction from the co-NExptime-hard inclusion problem of non-recursive Boolean datalog queries \cite{BenediktG10}. 
\end{proof}

\section{Outline and Perspectives}

In this paper, we have characterized boundedness in terms of FO-rewritability and chase termination, for the oblivious and semi-oblivious chase variants. 
We conclude with a discussion on the extent of our results to more powerful chase variants (i.e., which terminate at least when the semi-oblivious chase terminate).
Theorem \ref{thm-FO-depth-SO} suggests that whenever $\Sigma\in\forw$ if any such chase generates only terms of bounded \emph{frontier depth}  on all instances, 
then $\Sigma$ is bounded. 
We leave open the question to determine if for other chase variants, like the restricted and the core chases,  
boundedness is again the intersection of chase termination and FO-rewritability.

\paragraph{Acknowledgements. }
This work was supported by ANR projects CQFD  (ANR-18-CE23-0003),  DataCert (ANR-15-CE39-0009), DeLTA (ANR-16-
CE40-0007) and the CNRS-Momentum project Managing-Data.

\bibliographystyle{named}
\bibliography{bib}

\section*{Appendix}

This appendix contains the proofs that were omitted in the paper due to space limitation. Note that the proofs of Theorem \ref{thm-finitedepth-O} and Proposition \ref{prop:foafsupset} are provided in the paper, hence not recalled below. 

\paragraph{Proof of Proposition \ref{prop-DF}}
\emph{$\ruleset\in\cto$ iff $\DF\ruleset\in\cto$ and 
$\ruleset\in\bno$ iff $\DF\ruleset\in\bno$. 
}

\begin{proof}  The proposition is immediate since, for any instance $I$ and chase step $i$, $\ochase^{i}(I,\ruleset) =  \ochase^{i}(I,\DF\ruleset)$. 
\end{proof}

\paragraph{Remarks.} For the $\sochase$, only one direction holds true: if $\ruleset\in\ctso$ then $\DF\ruleset\in\ctso$ and 
if $\ruleset\in\bnso$ then $\DF\ruleset\in\bnso$. Note that the decomposition has no incidence 
on the FO-rewritability of $\Sigma$ since $\Sigma$ and $\DF{\Sigma}$  are logically equivalent.

\paragraph{Proof of Proposition \ref{prop-o-depth-rank}}
{\it If $\ruleset$ is a set of FE-rules then, for all instance $\I$ and term $\t$ in $\ochaseshort$,
 holds that $\valdepth{v}=\rank{v}$.}

 \medskip
 \noindent
Note that this proposition could also be stated for facts instead of terms.
 
 \begin{proof} 

By a straightforward induction on the rank of facts, we show that, for all $i \geq 0$ and fact $f$, if $\rank{f} = i$ then $\valdepth{f}=i$. 
The property obviously holds for $i = 0$. Let $i > 0$ and $\rank{f} = i$. By definition of rank, $f$ was produced from at least one fact $f'$ of rank $i-1$. By induction hypothesis, 
$\valdepth{f'}= i-1$, hence, by definition of existential depth, $f'$ contains a term $t$ with $\valdepth{t}=i-1$. Since all rules are FE-rules, $f$ contains at least one fresh variable (null) $v$, and, by definition of existential depth, $\valdepth{v} = 1 + (i-1) = i$. Hence, $\valdepth{f} = i$. 

Now, let $t$ be a term with rank $i$. If $i =0$, $t$ occurs in $I$ and $\valdepth{t} = 0$. Otherwise, $t$ has been generated in a fact $f$ of rank $i$. Since $\valdepth{f}=\rank{f}$, $\valdepth{f} = i$ and, by definition of existential depth, all terms generated in $f$ have existential depth $i$, in particular $t$. 
 \end{proof}

\paragraph{Proof of Lemma \ref{o-depth-preservation}}
{\it For any  embedding  $\map$ from $\I$ to $\I'$ and any $i \geq 0$, there exists an embedding $\map'\supseteq \map$ from $\ochasekshort{i}$
   to $\ochasek{\I'}{\ruleset}{i}$ which preserves the existential depth of terms, i.e., for every term $\t$ in $\ochaseshort$ it holds that
$\valdepth{v}=\valdepth{\map'(v)}$. }

\begin{proof}
By induction on $i$. If $i = 0$ then all terms have existential depth 0 in $\I$ and $\I'$, then for $\map' = \map$ the thesis follows. 
 Assume the property holds for $0\leq i <n$. Let $i = n$. 
By inductive hypothesis  there exists an embedding $\map':\ochasekshort{n-1}\rightarrow \ochasek{\I'}{\ruleset}{n-1}$ preserving the existential 
depth of terms.
Let $(\erule,\match)$ be any trigger of $\ochasekshort{n-1}$. 
We know that 
 $\map' \circ \match(\mathsf{body}(\erule))\subseteq\ochasek{\I'}{\ruleset}{n-1}$
 and $(\erule,\map'\circ\match)$ is a trigger
 of $\ochasek{\I'}{\ruleset}{n-1}$.
 Also, there exists a bijection $\rho_n$ from the fresh terms in $\safeo(\mathsf{head}(\erule))$ to the fresh terms in  $\map' \circ \safeo(\mathsf{head}(\erule))$ precisely defined as
 $\rho_n(z_{(\erule,\match)}) = z_{(\erule,\map' \circ\match)}$.
 Let $\map''=\map' \uplus\bigcup\rho_n$ 
be the natural extension of $\map'$ to all triggers that are performed to compute  $\ochasekshort{n}$.
Of course,  for every trigger $(\erule, \match)$ and term $v_\body\in\adom{\match(\mathsf{body}(\erule))}$ we have that
$\valdepth{v_\body}=\valdepth{\map''(v_\body)}$.
We want to show that $\map''$ also preserves the existential depth of fresh terms.
Consider now the rule application $(\erule,\map'' \circ \match)$.
Let $z$ be an existential variable of $\erule$.
Then, $\valdepth{z_{(\sigma,{\match})}}=1+\max\{\valdepth{v_\body}\}=1+\max\{\valdepth{\map''(v_\body})\}
=\valdepth{z_{(\sigma,{\map''\circ\match})}}$.
\end{proof}

\paragraph{Proof of Lemma \ref{so-depth-preservation}
}
{\it 
 For any embedding   $\map$ from $\I$ to $\I'$ and any $i \geq 0$, there exists an embedding $\map'\supseteq \map$ from $\sochasekshort{i}$
   to $\sochasek{\I'}{\ruleset}{i}$ which preserves the frontier depth of terms.
}

 \begin{proof} 
By induction on $i$. If $i = 0$ then all values have frontier depth 0 in $\I$ and $\I'$, then for $\map' = \map$ the thesis follows. 
 Assume the property holds for $0\leq i <n$. Let $i = n$. 
By inductive hypothesis, we know that there exists an embedding $\map':\sochasekshort{n-1}\rightarrow \sochasek{\I'}{\ruleset}{n-1}$ which preserves the frontier depth of values.
Let $(\erule, \match)$ be any trigger producing a new fact $f\in\sochasekshort{n}$.
Then $\map'\circ\match(\mathsf{body}(\erule))\subseteq  \sochasek{\I'}{\ruleset}{n-1}$.

Consider first the case where $\fr\erule=\emptyset$.
In this case $\frdepth f=1$ and any term $v\in \adom f$ is a fresh term $v=z_{(\erule,\emptyset)}$ generated from an existential variable $z\in\ex\erule$.
Thus $f\in\sochasek{\I'}{\ruleset}{n}$ as well and the embedding $\map'$ is the identity on the terms of $f$. Also, $f$ has frontier depth 1 in $\sochasek{\I'}{\ruleset}{n}$.

Now, if $\fr\erule\neq\emptyset$ we again distinguish two cases. If for all triggers of the form $(\erule, \match')$ applied to compute $\sochasek{\I'}{\ruleset}{n-1}$ we have that 
$\map\circ {\match}_{|\fr{\sigma}}\neq{\match'}_{|\fr{\sigma}}$ then the trigger $(\erule, \map'\circ\match)$  has not yet been applied in $\ochasek{\I'}{\ruleset}{n-1}$.
So, we define  $\map''\supseteq \map'$ to be such that $\map''(z_{({\sigma,\match_{|\fr{\erule}}})})=z_{({\sigma,\map'\circ\match_{|\fr{\erule}}})}$ for every $z\in\fr\erule$.
Otherwise,  there is a trigger $(\erule, \match')$ such that $\map\circ {\match}_{|\fr{\sigma}}={\match'}_{|\fr{\sigma}}$ applied to compute $\sochasek{\I'}{\ruleset}{n-1}$ which makes $(\erule, \map'\circ\match)$ producing the same result as $(\erule, \match')$.
In this case we define $\map''\supseteq \map'$ to be such that  $\map''(z_{({\sigma,\match_{|\fr{\erule}}})})=z_{({\sigma,\match'_{|\fr{\erule}}})}$ for all $z\in\fr\erule$. 
To conclude, we have that $\frdepth{z_{({\sigma,{\match_{|\fr{\erule}}}})}}=1+\max\{\frdepth{{v}}~|~ v\in\match({\fr\erule})\}=1+\max\{\frdepth{\map''({v}})~|~ v\in\match({\fr\erule})\}
=\frdepth{\map''(z_{({\sigma,{\match_{|\fr{\erule}}}})})}$. 
 \end{proof}

\paragraph{Proof of Theorem \ref{thm-finitedepth-SO}}
{\it When $\ruleset\in\ctso$ there exists a constant $\kd$ such that for all instance $I$, the frontier depth of a term in $\sochaseshort$ is bounded by $\kd$.}

\begin{proof}
If $\ruleset$ is in $\ctso$, the $\sochase$ terminates on the critical instance. We take for $\kd$ the smallest rank such that 
$\sochasek{\Ia}{\ruleset}{\kd}=\sochasek{\Ia}{\ruleset}{}$. 
Every instance can be embedded into the critical instance.
Hence, by Lemma  \ref{so-depth-preservation}
 the frontier depth of the terms in $\sochaseIS{\I}{\ruleset}$ is bounded by the frontier depth of $\sochaseIS{\Ia}{\ruleset}$, which is itself bounded by $\kd$.
\end{proof}

\paragraph{Proof of Proposition \ref{prop-FO-atomicfull-instance-datalogheads}}
\emph{Let $\ruleset$ be a ruleset and $\mathsf{HD}_\ruleset$ be the full-atomic queries given by heads of the datalog rules in $\DF\ruleset$.
Then,
 $\ruleset\in \forwaf $ if and only if
  $\ruleset\in \forw^{\mathsf{HD}_\ruleset}$.}

\begin{proof}
Since $\mathsf{HD}_\ruleset$ is a particular set of full-atomic queries,  $\ruleset\in \forwaf $ implies $\ruleset\in \forw^{\mathsf{HD}_\ruleset}$.
For the other direction, first note that $\ruleset$ and $\DF\ruleset$ are equivalent sets of rules, hence they behave similarly with respect to first-order rewritability. 
Specifically, for any CQ $Q$ and set of rules $\ruleset$, $(Q, \ruleset)$ is  FO-rewritable iff $(Q, \DF\ruleset)$ is FO-rewritable. Hence, we conveniently consider in the following 
that $\ruleset$ is in the form of $\DF\ruleset$.  

When a query is rewritten, some answer variables may be made equal. Hence, we slightly generalize the notion of query $Q(x_1,...,x_k)$ by allowing 
 to equate some answer variables, which is represented by assigning  to $Q$ the partition $P_Q$ on $\{1,\ldots,k\}$ associated with answer variable equality, i.e., $i$ and $j$ are in the same class of $P_Q$ iff the \emph{ith} and  \emph{jth} answer variables of $Q$ are the same. 
 Given a class $C$ in $P_Q$, we denote by $x_C$ the answer variable associated with $C$.
Then, the full-atomic query given by an atom has exactly the same arity as this atom, for instance the query associated with $p(x,x,y)$ is $Q(x_{\{1,2\}},x_{\{1,2\}},x_3) = p(x_{\{1,2\}},x_{\{1,2\}},x_3)$, with $P_Q = \{\{1,2\},\{3\}\}$, and not a query of the form $Q(x_1,x_2) = p(x_1,x_1,x_2)$. 
Given partitions $P_1$ and $P_2$ on $\{1,\dots, k\}$, we note $P_1 \sqsubseteq P_2$ if $P_1$ is thinner than $P_2$, i.e., for each class $C \in P_1$, there is a class $C' \in P_2$ with $C \subseteq C'$. The $\sqsubseteq$ relation organizes the set of partitions of $\{1,\dots, k\}$ into a lattice. As usual, we denote by $P_1 \lor P_2$ the upper bound of $P_1$ and $P_2$ in this lattice. 

We recall that $(Q,\ruleset)$, with $Q$ a CQ, is FO-rewritable iff there is a UCQ-rewriting of $Q$, i.e., a finite set $\mathcal Q$ of CQs  such that, for any instance $I$, the set of certain answers to $Q$ on $(I, \ruleset)$ is exactly the set of answers to the UCQ obtained from $\mathcal Q$ on $I$.\footnote{As already noticed, the equivalence between the rewritability into a union of CQs and first-order rewritability follows from the (Finite) Homomorphism Preservation Theorem \cite{DBLP:journals/jacm/Rossman08}.} 
Each CQ in $\mathcal Q$ can be obtained from $Q$ and $\ruleset$ by a finite rewriting sequence, based on so-called piece-unifiers (see e.g., \cite{DBLP:journals/semweb/KonigLMT15} for definitions). More precisely, each query $Q_{i+1}$ in a rewriting sequence is obtained from the preceding query $Q_i$, and a piece-unifier $u$ of $Q_i$ with a rule 
$\erule \in \Sigma$ that unifies a non-empty subset of $Q_i$ with a subset of $\erule$'s head while satisfying conditions concerning existential variables in $\erule$. 
In particular, an answer variable of $Q_i$ cannot be unified with an existential variable of $\erule$. 
The following property holds: for any $(I, \Sigma)$ and CQ $Q$, a tuple of constants $(a_1,\ldots,a_k)$ is a certain answer to $Q$ on $(I, \ruleset)$ iff there is a finite rewriting sequence from $Q$ to a CQ $Q'$ such that $(a_1,\ldots,a_k)$ is an answer to $Q'$ on $I$. 
Also note that for any CQ $Q_i$ obtained from a CQ $Q$ by a rewriting sequence, $P_{Q} \sqsubseteq P_{Q_i}$ holds ($P_{Q}$ is thinner than $P_{Q_i}$). 

Now, assume $\ruleset\in \forw^{\mathsf{HD}_\ruleset}$ and let $Q$ be a full-atomic query. If $Q$ is not unifiable with a datalog rule from $\DF\ruleset$, its UCQ-rewriting is $Q$ itself, because a full-atomic query is not unifiable with an FE rule. Otherwise, let $Q_h$ be the full-atomic query associated with any datalog rule head unifiable with $Q$ by a unifier $u$. One has $u(Q) = u(Q_h)$. If all such $(u(Q_h), \ruleset)$ are FO-rewritable, we obtain that $(Q,\ruleset)$ is FO-rewritable, as the union of the UCQ-rewritings of all $u(Q_h)$ yields a suitable UCQ-rewriting of $Q$. We will show the following property (P1): let $Q$ and $Q_s$ be two full-atomic queries with the same predicate such that $P_Q \sqsubseteq P_{Q_{s}}$; if $(Q, \ruleset)$ is FO-rewritable then $(Q_s, \ruleset)$ also is. By hypothesis, each $(Q_h, \ruleset)$ is FO-rewritable, hence (P1) implies that, for any substitution $u$, $(u(Q_h), \ruleset)$ is also FO-rewritable, which will conclude the proof. 

It remains to prove (P1). We prove a preliminary lemma (L):  let $Q$ and $Q_s$ be two full-atomic queries with the same predicate  such that $P_Q \sqsubseteq P_{Q_{s}}$; then, for any rewriting sequence of length $l$ leading from $Q$ to a query $Q^l$, there is a rewriting sequence of the same length leading from $Q_s$ to a query $Q_s^l$, such that $P_{Q_s^l} = P_{Q_l} \lor P_{Q_s}$ (where $\lor$ is the  upper bound in the partition lattice), and, given $s'$ the substitution of the answer variables in $Q^l$ by the  answer variables in $Q_s^l$ associated with $P_{Q^l} \sqsubseteq P_{Q_s^l}$, it holds that $s'(Q^l) = Q_s^l$, up to a bijective renaming of non-answer variables. 
Let us now prove (P1). Let $Q$ and $Q_s$ be two full-atomic queries on the same predicate of arity $k$ such that $P_Q \sqsubseteq P_{Q_{s}}$ and $(Q, \ruleset)$ is FO-rewritable. The FO-rewritability of $(Q, \ruleset)$ is equivalent to the following statement: there is an integer $b$ such that for any instance $I$ and any tuple of constants $(a_1,\ldots, a_k)$, it holds that $(a_1,\ldots, a_k)$ is a certain answer to $Q$ on $(I, \ruleset)$ if and only if there is a query $Q^b$ obtained by a rewriting sequence from $Q$ of length less than $b$, with $(a_1,\ldots, a_k)$ is an answer to $Q^b$ on $I$. We prove that $(Q_s,\ruleset)$ is FO-rewritable by such a statement. 
Let $I$ be any instance and $(a_1,\ldots, a_k)$ be a certain answer to $Q_s$ on $(I, \ruleset)$. 
There is thus a homomorphism $h$ from $Q_s$ to $\anychase(I, \ruleset)$ that maps its answer variable tuple to $(a_1,\ldots, a_k)$. 
Given $s$ the homomorphism from $Q$ to $Q_s$, it holds that $h \circ s$ is a homomorphism from $Q$ to $\anychase(I, \ruleset)$ that maps its answer variable tuple  to $(a_1,\ldots, a_k)$.
Since $Q$ is FO-rewritable, there is a query $Q^b$ obtained by a rewriting sequence of length less than $b$ such that $(a_1,\ldots, a_k)$ is an answer to $Q^b$ on $I$. Let $h^b$ be a homomorphism from $Q^b$ to $I$ yielding this answer. 
Let $P_a$ be the partition on $\{1,...,k\}$ associated with the equality of terms in $(a_1,\ldots, a_k)$. We have $P_{Q^b} \sqsubseteq P_a$ and 
$P_{Q_s} \sqsubseteq P_a$. By Lemma (L),   there is a query $Q_s^b$ obtained from $Q_s$ with a rewriting sequence of length less than $b$, such that 
(1) $P_{Q_s^b} = P_{Q^b} \lor P_{Q_s}$, and (2) $s'(Q^b) = Q_s^b$, with $s'$ the substitution associated with $P_{Q_b} \sqsubseteq P_{Q_s^b}$. From (1), we have  $P_ {Q_s^b} \sqsubseteq P_a$. Hence, the homomorphism $h^b$ from $Q^b$ to $I$ can be written $h' \circ s'$, where $h'$ is a homomorphism from $Q_s^b$ to $I$ mapping its answer tuple to $(a_1,\ldots,a_k)$. The converse direction (``if there is $Q_s^b$ obtained by a rewriting sequence from $Q_s$, of length less than $b$, with $(a_1,\ldots, a_k)$ is an answer to $Q^b_s$ on $I$, then $(a_1,\ldots, a_k)$ is a certain answer to $Q_s$ on $(I, \ruleset)$'') holds because of the soundness of query rewriting based on piece-unifiers. 

\end{proof}

\paragraph{Proof of Lemma \ref{lem-FORWAF-rank}}
{\it  If $\ruleset\in\forwaf$ there is a constant $\kaf$ such that, for any instance $\I$ and fact $f$ such that $\adom f \subseteq \adom\I$, when $f \in \ochaseshort$ it holds that  $\rank{f} \leq \kaf$.}

\begin{proof} Assume $\ruleset\in\forwaf$. 
We take for $\kaf$ the maximal number of breadth-first rewriting steps necessary to obtain a UCQ-rewriting of 
a full-atomic query (we refer here to the breadth-first rewriting based on aggregated piece-unifiers, see \cite{DBLP:conf/rr/KonigLMT13}; 
this query rewriting technique ensures the following property: for any $(I, \Sigma)$ and any CQ $Q$, for any $k$, the set of answers to $Q$ on $\ochasekshort{k}$ is equal to the set of answers to $Q_k$ on $\I$, where $Q_k$ is the UCQ-rewriting of $Q$ with $\ruleset$ obtained by $k$ breadth-first rewriting steps).\footnote{Alternatively, we could rely on the bound 
given by the bounded derivation-depth property (BDDP)
\cite{Cali2009}. A ruleset  $\ruleset$ satisfies the BDDP if for all Boolean CQ $Q$, there is an integer $k$ such that, for all instance $\I$, it holds that $\I, \Sigma  \models Q$ iff $\ochasekshort{k} \models Q$. It has been several times remarked that BDDP is equivalent to UCQ-rewritability, hence to FO-rewritability. } 

 We know that the number of full-atomic queries to be considered is finite and by Proposition \ref{prop-FO-atomicfull-instance-datalogheads} can be even bounded by the number of non-isomorphic heads of datalog rules. 
By the properties of breadth-first query rewriting, we know that for any instance $\I$ and 
 any full-atomic query $\query$, the certain answers to $\query$ on $(I, \Sigma)$ are exactly the answers to $\query$ on  $\ochasekshort{\kaf}$. 
 We can identify an answer $(\const_1,\dots,\const_n)$ to $\query=p(x_1,\dots, x_n)$ with the fact (having only constants)  $p(\const_1,\dots,\const_n)$. 
 Then we have that  $(\const_1,\dots,\const_n)$ is a certain answer to $\query$ iff $ p(\const_1,\dots,\const_n) \in \ochasekshort{\kaf}$. 
  Now, the statement of the lemma considers more generally facts $f$ with $\adom f \subseteq \adom\I$ which could contain also existentially quantified variables. 
  Observe however that, 
for any instance  $\I$, let $\freeze{}$ be a bijective renaming of the variables of $\I$ by constant values that do not appear in $\I$,  
then, for all $i \geq 0$, there is an isomorphism from $\ochasekshort{i}$ to $\ochasek{\freeze{\I}}{\ruleset}{i}$
 that preserves the rank of facts.
 Because of this,
 let $f$ and $\I$ be any fact and instance, 
  we have that $f \in \ochasekshort{i}$ (for any $i$) iff $\freeze f$ $\in \ochasek{\freeze{\I}}{\ruleset}{i}$. Since 
 $\freeze f$ contains only constants we can conclude.
\end{proof}

\paragraph{Proof of Lemma \ref {lem-FORW-rank}}
{\it If $\ruleset\in\forw$ there is a constant $\kfo$ such that, for any instance $\I$ and any trigger
 $(\erule,\match)$ from $\sochaseshort$
  with $\match(\fr{\erule}) \subseteq \adom{\I}$,
there is also a trigger $(\erule,\match')$ from 
$\sochaseshort $ such that $\match'_{|\fr{\erule}}=\ \match_{|\fr{\erule}}$ and  $\rank{f} \leq \kfo$ for all $f\in\match'(\mathsf{body}(\erule))$.}

\begin{proof} Assume $\ruleset\in\forw$.  
We denote by  $\query_{\mathsf{body}(\erule)}$ any conjunctive query whose atoms correspond to the atoms of $\mathsf{body}(\erule)$, for $\erule\in\ruleset$, and all variables are existentially quantified except for those in $\fr\erule$, which are the answer variables.
We know that the number of such queries is bounded by the cardinal of $\ruleset$.
We take for $\kfo$ the maximal number of breadth-first rewriting steps necessary to obtain a UCQ-rewriting from any $\query_{\mathsf{body}(\erule)}$ query. 

By the properties of breadth-first query rewriting (based on aggregated piece-unifiers \cite{DBLP:conf/rr/KonigLMT13}), we know that, for any instance $\I$, the certain answers to $\query_{\mathsf{body}(\erule)}$ on $(I, \ruleset)$ are exactly the answers to $\query_{\mathsf{body}(\erule)}$ on  $\ochasekshort{\kfo}$ and therefore on $\sochasekshort{\kfo}$. 
By definition of query-answer, there is a homomorphism $\match'$ from $\query_{\mathsf{body}(\erule)}$ to $\sochasekshort{\kfo}$ which maps each answer variable 
(recall, originated from a frontier variable) to a constant of $\I$ and each existentially quantified variable to a term of $\sochasekshort{\kfo}$. 

Let  $(\erule,\match)$ any trigger  from  $\sochaseshort$ with $\match(\fr{\erule}) \subseteq \adom{\I}$.
 Assume first that all values in $\I$ are constants.  
According to the previous observation, 
we know that 
there is $\match'$ from  $\query_{\mathsf{body}(\erule)}$ to $\sochasekshort{\kfo}$ with 
$\match'_{|\fr{\erule}}=\ \match_{|\fr{\erule}}$
and therefore a trigger $(\erule,\match')$ as desired.
Now, for the case where $\I$ also contains existentially quantified variables,  let $\I$ be an instance and let $\freeze{}$ be a bijective renaming of the variables of $\I$ by constant values that do not appear in $\I$ then, for all $i \geq 0$, there is an isomorphism from $\sochasekshort{i}$ to $\sochasek{\freeze{\I}}{\ruleset}{i}$
 that preserves the rank of facts.
 This implies that if $\I$ is any instance then
for any $f$ and $i$ holds
 $f \in \sochasekshort{i}$ iff freeze($f$) $\in \sochasek{\freeze{\I}}{\ruleset}{i}$ and we conclude.
\end{proof}

A closer look at our proof actually shows that FO-rewritability with respect to queries associated with rule bodies  is sufficient to derive Lemma \ref{lem-FORW-rank}. 

\paragraph{Proof of Theorem \ref{thm-FO-depth-O}}
{\it If $\ruleset\in\forwaf$ then 
 for all instance $\I$ and fact $f\in\ochaseshort $ 
 we have that 
 $\rank{f}\leq \valdepth{f}\times (\kaf + 1)+\kaf $
 with 
 $\kaf$ the bound provided by Lemma \ref{lem-FORWAF-rank}.}

\begin{proof}
We first show that since $\ruleset\in\forwaf$ then 
for all instance $\I$ and term $\t\in\adom{\ochaseshort}$ it holds that
  $\rank\t~\leq~\valdepth{\t} \times (\kaf + 1)$, where we recall that $\rank\t$ is the rank where $\t$ is introduced. 
  
By induction on the existential depth of $\t$. If $\valdepth{\t}=0$ then $\t\in\adom{\I}$ and thus $\rank{\t}=0$ also. 
Assume the property holds for $0\leq\valdepth{\t}\leq n$.
We show that it holds for $\valdepth{\t}= n+1$.
Let $(\erule,\match)$ be the trigger that generates $\t$.
Then, for all $\t_B\in \adom{\match(\mathsf{body}(\erule))}$, we know that $\valdepth{\t_B}\leq n$.
By inductive hypothesis, $\rank{\t_B}\leq n \times (\kaf +1) $. 
Since $\ruleset\in\forwaf$,
we can apply Lemma \ref{lem-FORWAF-rank} using as instance $\ochasekshort{k}$ where $k=n\times (\kaf+1)$. 
Hence for all $f_B \in \match(\mathsf{body}(\erule))$ it holds that  $\rank{f_B}\leq 
k+ \kaf$. 
 Thus $\rank{\t}\leq 
k\;+\;\kaf \;+\; 1
= (n + 1) \times (\kaf + 1) =
\valdepth{\t}\times (\kaf + 1)$.
To conclude the proof, since any fact $f\in\ochaseshort$ contains only terms $\t$ with $\rank{\t}~\leq~\valdepth{\t}\times (\kaf +1)$, we apply Lemma \ref{lem-FORWAF-rank} and we obtain $\rank{f}~\leq~\max\{\;\valdepth{\t}\;|\;\t\in\adom f\;\} \times (\kaf + 1) \; + \;  \kaf$.
\end{proof}

\paragraph{Proof of Theorem 
 \ref{thm-FO-depth-SO}}
{\it If $\ruleset\in\forw$
 then 
 for all instance $\I$ and fact $f\in\sochaseshort $ 
  we have that 
 $\rank{f}\leq \frdepth{f}\times (\kfo + 1)+\kfo $
 where 
 $\kfo$ is the bound provided by Lemma \ref{lem-FORW-rank}.
}

\begin{proof}

We first show that since $\ruleset$ is FO-rewritable, then 
for all instance $\I$ and term $\t\in\adom{\sochaseshort}$ it holds that
  $\rank{\t}~\leq~\frdepth{\t} \times (\kfo + 1)$.
By induction on the frontier depth of $\t$.
If $\frdepth{\t}=0$ then $\t\in\adom{\I}$ and thus $\rank{\t}=0$ also. 
Assume the property holds for $0\leq\frdepth{\t}\leq n$.
We show that it holds for $\frdepth{\t}= n+1$.
Let $(\erule,\match)$ be the trigger that generates $\t$.
By definition of  frontier depth, for all $\t_{fr}\in \adom{\match(\fr{\erule})}$, we know that $\frdepth{\t_{fr}}\leq n$.
By inductive hypothesis, $\rank{\t_{fr}}\leq n \times (\kfo +1)$.
Since $\ruleset\in\forw$, 
we can apply Lemma \ref{lem-FORW-rank} using as instance $\sochasekshort{k}$ where $k=n \times (\kfo +1)$. This gives us $\rank{\match(\mathsf{body}(\erule))}\leq 
k  \;+\; \kfo$. 
Thus $\rank{\t}\leq 
n\times (\kfo+1)\;+\; \kfo \;+\; 1
= (n + 1) \times (\kfo + 1) =
\frdepth{\t}\times (\kfo + 1)$.

Since any fact $f\in\sochaseshort$ contains only terms $\t$ with $\rank{\t}~\leq~\frdepth{\t}\times (\kfo +1)$, we use again Lemma \ref{lem-FORW-rank}, and obtain $\rank{f}~\leq~\frdepth{f}
\times (\kfo +1) + \kfo$.
\end{proof}

\paragraph{Proof of Theorem \ref{thm-bn-foaf-cto}}
{\it $\bno=\forwaf\cap \cto $.}

\begin{proof} ($\Rightarrow$) If $\ruleset\in \bno $ there is $k$ such that, for all $\I$,  $\ochasekshort{k}=\ochaseshort$.
Hence, there is $k$ such that for all $\I$ and all Boolean CQ $\query$ we have $\ruleset, \I \models \query$ (i.e., $\ochaseshort  \models \query$  iff $\ochasekshort{k} \models \query$ iff $\I \models \query^k$, where  $\query^k$ is the query obtained by $k$ steps of breadth-first rewriting from $\query$ and $\ruleset$, which implies the FO-rewritability of $\ruleset$. To conclude,  $\bno\subseteq \cto$ follows by definition.

 ($\Leftarrow$)
By Theorem \ref{thm-finitedepth-O}  and  \ref{thm-FO-depth-O}.
\end{proof}

\paragraph{Proof of Theorem \ref{thm-bn-fo-ctso}}
{\it $\bnso=\forw\cap \ctso $.}

\begin{proof} Identical to proof of Theorem \ref{thm-bn-foaf-cto} for the direct sense. The other direction holds by Theorem \ref{thm-finitedepth-SO}  and  \ref{thm-FO-depth-SO}.
\end{proof}

The next proposition leads to conclude that membership to $\bnso$, $\ctso$ and $\forw$ remains undecidable for FE-rules.

\begin{proposition}
\label{prop-datalog-fully-exist}
There is a translation $\psi$  from any KB $(\I, \ruleset)$ on a vocabulary $\Voc$, where $\ruleset$ is a set of existential rules, to a KB $(\psi(\I), \psi(\ruleset))$ on a vocabulary $\psi(\Voc)$, where $\psi(\ruleset)$ is a set of FE-rules, such that:
\\ (1) $\psi$ is injective, and 
\\(2) $\sochaseshort$ and $\sochaseIS{\psi(\I)}{\psi(\ruleset)}$ have the same rank, and 
\\(3) for any instance $I'$ on $\psi(\Voc)$, there is an instance $\psi(\I)$ such that $\sochaseIS{I'}{\psi(\ruleset)}$ and $\sochaseIS{\psi(\I)}{\psi(\ruleset)}$ have the same rank. 
\end{proposition}

The proposition leads directly to the undecidability of $\bnso$ and $\ctso$for FE-rules. Concerning the undecidability of $\forw$ for FE-rules, we take for $\ruleset$ a set of datalog rules. Then $\ruleset$ is in $\ctso$ if and only if $\psi(\ruleset)$ is in $\ctso$. Since every datalog set  is $\ctso$, $\psi(\ruleset)$ is also $\ctso$. Now, consider the (undecidable) problem of whether $\ruleset$ is (uniformly) bounded. We have that $\ruleset$ is bounded iff $\psi(\ruleset)$ is bounded, which amounts to asking if $\psi(\ruleset)$ is FO-R (as we already know it is in $\ctso$).


%

\begin{proof} (of proposition \ref{prop-datalog-fully-exist}). 
Take a vocabulary $\Voc=(\Pred,\Const)$ and define the set $\Pred^+$ where each predicate $p\in\Pred$ of arity $k$ is replaced by a predicate $p^+$ of arity $k+1$. 
Let $\psi$ be a transformation defined as follows. First, $\psi(\Voc) = (\Pred^+,\Const)$. 
Then, given an atom $\atom=p(\t_1,\dots,\t_k)$ then $\psi(\atom)=p^+(\t_1,\dots,\t_k,z_\atom)$ where $z_\atom$ is a fresh variable.   
Let  $\erule=\body_1(\bar x,\bar y)\dots \body_n(\bar x,\bar y)\rightarrow \exists \bar z \head_1(\bar x,\bar z),\dots,\head_m(\bar x,\bar z)$  be a rule, 
then
$\psi(\erule)=\psi(\body_1),\dots,\psi(\body_n)\rightarrow \exists z_{\head_1}\dots z_{\head_m}\psi(\head_1)\dots \psi(\head_m)$.
Finally, $\psi(\I)=\bigcup_{\atom\in\I}\psi(\atom)$ and $\psi(\ruleset)=\bigcup_{\erule\in\ruleset}\psi(\erule)$.   

Obvliously, $\psi$ is injective (Point (1)). 
  
To prove the point (2), we show that for each fact $f=p(\t_1,\dots,\t_n)\in\sochasekshort{i}$ 
generated by a trigger $(\erule,\match)$  it holds $p^+(\t_1,\dots,\t_n,z_{(\psi(\erule),\match_{|\fr{\erule}})})\in\sochasek{\psi(\I)}{\psi(\ruleset)}{i}$, and vice-versa. 

We focus on the direction $\Rightarrow)$ as the direction $\Leftarrow)$ is similar.
By induction on the rank $i$ of the $\sochase$.
If $i=0$ then by definition $f \in\I$ implies $\psi(f)\in\psi(\I)$. Assume that the property holds for  $0\leq i\leq n$. We show that it holds for rank $n+1$.
Let $f$ be any atom of rank $n+1$ produced by the trigger $(\erule,\match)$. 
This means that for all body atom $f_B=p(x_1,\dots,x_k)\in \mathsf{body}(\erule)$ we know that 
$\match(f_B)=p(\t_1,\dots,\t_n)\in\sochasekshort{n}$. Hence, by induction 
$\psi(\match(f_B))=
p^+(\t_1,\dots,\t_n,z^+)\in\sochasek{\psi(\I)}{\psi(\ruleset)}{n}$ where $z^+=z_{\match(f_B)}$ if $\match(f_B)\in\I$, or $z^+=z_{(\psi(\erule'),\match'_{|\fr{\erule'}})}$ if $\match(f_B)$ has been generated by a trigger $(\erule',\match')$.
Then, the trigger $(\psi(\sigma),\match\cup\bigcup_{f_B}\{z_{f_B}\mapsto z^+\})$
is applicable and produces 
$\psi(\match(f))$. 
Since $f$ is of rank $n+1$ there does not exist another rule application that could have generated the same atom at a previous rank, and the same holds for its image.

For point (3), we build a transformation $\phi$ from any instance $\I'$ on $\psi(\Voc)$ to an instance $\I$ on $\Voc$
such that $\sochaseIS{I'}{\psi(\ruleset)}$  and $\sochaseshort$ have the same rank. 
[Note that the proof does not follow exactly point (3) here: we consider directly $\I$ instead of $\psi(\I)$]. 
 
 The transformation $\phi$ assigns to each atom $\atom=p^+(\t_1,\dots,\t_k,z)$ on $\psi(\Voc)$ the atom $\phi(\atom)=p(\t_1,\dots,\t_k)$.   
Let  $\I=\phi(\I')=\bigcup_{\atom\in\I'}\phi(\atom)$.
We show that  for each  fact $f=p^+(\t_1,\dots,\t_n,z)\in\sochasek{\I'}{\psi(\ruleset)}{i}$ 
generated by a trigger $(\psi(\erule),\match)$,  it holds that $p(\t_1,\dots,\t_n)\in\sochasek{\I}{\ruleset}{i}$.
By induction on the rank $i$ of the $\sochase$.
If $i=0$ then by definition $f \in\I'$ implies $\phi(f)\in\I$. Assume that the property holds for  $0\leq i\leq n$. We show that it holds for rank $n+1$.
Let $f$ be any atom of rank $n+1$ produced by the trigger $(\psi(\erule),\match)$. 
This means that for all body atom $f_B=p^+(v_1,\dots,v_k,z)\in \mathsf{body}(\psi(\erule))$ we know that 
$\match(f_B)\in\sochasek{\I'}{\psi(\ruleset)}{n}$. Hence, by induction 
$\phi(\match(f_B)) \in\sochasekshort{n}$. Then, the trigger $(\erule,\match')$
 (where $\match'$ is the appropriate restriction of $\match$) is applicable and produces 
$\phi(\match(f))$.  Since $f$ is of rank $n+1$  there does not exist another rule application that could have generated the same atom at a previous rank; as the chase is semi-oblivious,  and the 
last component of a predicate never occurs in the frontier,
$\phi(\match(f))$ is also of rank $n+1$.

So, let an instance $\I$ on $\Voc$. Let us note that $\phi(\psi(I))=I$. By what precedes $\sochaseshort$ and $\sochaseIS{\psi(I)}{\psi(\ruleset)}$ have the same rank.
Now, let an instance $\I'$ on $\psi(\Voc)$ and $I =\phi(I')$: by what precedes the rank of $\sochaseshort$ is at least the rank of  $\sochaseIS{\I'}{\psi(\ruleset)}$. 
Furthermore, by embedding $\psi(I)$ in $\I'$, we can by using similar arguments prove that  the rank of  $\sochaseIS{\I'}{\psi(\ruleset)}$ is at least the rank of $\sochaseshort$. So,   $\sochaseIS{\I'}{\psi(\ruleset)}$  and $\sochaseshort$ have the same rank.

\end{proof}

\paragraph{Proof of Theorem \ref{th-k-bounded}}
{\it The \emph{$k$-boundedness} problem is:
\begin{itemize}
\item  in 2Exptime on general existential rules for the $\ochase$ and $\sochase$;
\item co-NExptime-complete on datalog;
\item in co-NExptime on FE-rules for the $\ochase$.
 \end{itemize}}

\begin{proof}
The upper bound results mostly come from \cite{RR18} \cite{DBLP:journals/corr/abs-2004-10030}.
Indeed, from these papers, a ruleset is $k$-bounded for the $\soochase$ iff  the $\soochase$ stops within $k$ steps for  instances of size at most $b^{k+1}$, with $b$ the maximum number of atoms in a rule body.
So, to  disprove $k$-boundedness, it suffices to guess a  breadth-first derivation from an instance of size at most $b^{k+1}$ to a fact of rank $k+1$, which can be done in NExptime for datalog.
This  gives also the 2Exptime upper bound by checking exhaustively for  each of these instances that there is no  breadth-first derivation of  depth $k+1$.
When the ruleset is fully existential, it can be proven  that a rule set is  not $k-$bounded  for  the $\ochase$ iff  there exists  a non-necessarly breadth-first partial derivation of depth $k+1$ from the critical  instance.
So, by guessing such a derivation, we get also a co-NExptime decision procedure on FE-rules  for the $\ochase$.
See  also \cite{gallois:tel-02445754} for detailed proofs of the above results. 

Co-NExptime-hardness of
$k$-boundedness for datalog rules is proven by reduction from the co-NExptime-hard inclusion problem of non-recursive Boolean datalog queries \cite{BenediktG10}. 
 Let $Q_1, Q_2$ two non-recursive Boolean datalog queries and $P_0$ (resp. $P_0^2$) their respective distinguished 0-ary predicate.
 As they are non-recursive,  $Q_1$ (resp. $Q_2$) is $k_1$-(resp.$ k_2$) bounded with $k_1$ (resp. $k_2$) the number of predicates in $Q_1$ (resp. $Q_2$).
 Let $p=max(k_1,k_2)+2$. Let us note that the size of $p$ encoded in unary is  bounded by the size of $(Q_1, Q_2)$.
Let us define a new ruleset $Q'_1 \cup Q'_2$:
 $Q'_1$ (resp.  $Q'_2$) is obtained from $Q_1$ (resp. $Q_2$ ) by adding  0-ary predicates $P_i$ and rules $P_{i-1}\rightarrow P_{i} $  (resp. $P_0 \rightarrow P_i$) with $1 \le i \le p$.
The size of $Q'_1 \cup Q'_2$ is linear w.r.t. the size of $(Q_1, Q_2)$. We will prove that $Q'_1 \cup Q'_2$  is $p-1$-bounded iff $Q_1$ is contained in $Q_2$

 Let us  first suppose that $Q_1$ is contained in $Q_2$. Let $\I$ be any instance. 
 If $P_0$  can be derived from  $(\I,Q_2)$,  all the $P_i$  are generated in at most  $k_2+1$  steps and so  the breadth-first chase 
 for $(\I, Q'_1 \cup Q'_2)$ stops after $max(k_1,k_2+1)$  steps.
 Otherwise, $P_0$  can neither be derived from  $(\I,Q_1)$ and the breadth-first chase for $(\I, Q'_1 \cup Q'_2)$ stops after 
 $max(k_1,k_2)$  steps.
 So, in both cases, $Q'_1 \cup Q'_2$  is $(p-1)$-bounded.  

If  $Q_1$ is not contained in $Q_2$, there exists $\I$ such that $P_0$ can be derived from $ (\I,Q_1)$ whereas $P_0^2$ can not be derived from  $(\I,Q_2)$.
 As $P_0^2$ is not generated by the $Q'_2$ part,  $P_k$ will be generated by the $Q'_1$ part, and so  the breadth-first chase for $(\I, Q'_1 \cup Q'_2)$ will need at least $p$ steps.
   
So $Q'_1 \cup Q'_2$  is $p-1$-bounded iff $Q_1$ is contained in $Q_2$. 
\end{proof}

\end{document}